\documentclass{article}
\usepackage[T1]{fontenc}
\usepackage[utf8]{inputenc}
\usepackage{authblk}
\usepackage{float}

\usepackage[square, sort, numbers]{natbib}
\usepackage{xcolor,soul,framed} 

\colorlet{shadecolor}{yellow}

\usepackage[cmex10]{amsmath}
\usepackage{array}
\usepackage{mdwmath}
\usepackage{mdwtab}
\usepackage{eqparbox}
\usepackage{url}
\usepackage{pifont}
\usepackage{color}
\usepackage{tikz}
\usepackage{pgfplots}
\usepackage{subcaption}

\newcommand{\E}{\mathbb{E}}
\newcommand{\Var}{\mathrm{Var}}

\usepackage{amssymb, amsthm}  
\usepackage{bm}

\DeclareSymbolFont{extraup}{U}{zavm}{m}{n}
\DeclareMathSymbol{\varheart}{\mathalpha}{extraup}{86}
\DeclareMathSymbol{\vardiamond}{\mathalpha}{extraup}{87}

\newtheoremstyle{mystyle}
  {}
  {}
  {\itshape}
  {}
  {\bfseries}
  {.}
  { }
  {\thmname{#1}}

\theoremstyle{mystyle}

\newtheorem*{thm*}{Theorem}
\newtheorem{thm}{Theorem}
\newtheorem*{corol*}{Corollary}
\newtheorem{corol}{Corollary}
\newtheorem{lemma}{Lemma}

\theoremstyle{definition}
\newtheorem*{definition}{Definition}

\usetikzlibrary{arrows,positioning} 
\tikzset{
    >=stealth',
    punkt/.style={
           rectangle,
           rounded corners,
           draw=black, very thick,
           text width=6.5em,
           minimum height=2em,
           text centered},
    pil/.style={
           ->,
           thick,
           shorten <=2pt,
           shorten >=2pt,}
}

\begin{document}
\title{LS-SVR as a Bayesian RBF network}

\author[1]{Diego P. P. Mesquita}
\author[2]{Luis A. Freitas}
\author[2]{Jo\~ao P. P. Gomes}
\author[2]{C\'esar L. C. Mattos}
\affil[1]{Aalto University}
\affil[2]{Federal University of Cear\'a }



\maketitle

\begin{abstract}
We show theoretical similarities between the Least Squares Support Vector Regression (LS-SVR) model with a Radial Basis Functions (RBF) kernel and maximum a posteriori (MAP) inference on Bayesian RBF networks with a specific Gaussian prior on the regression weights. Although previous works have pointed out similar expressions between those learning approaches, we explicit and formally state the existing correspondences. We empirically demonstrate our result by performing computational experiments with standard regression benchmarks. Our findings open a range of possibilities to improve LS-SVR by borrowing strength from well-established developments in Bayesian methodology.
\end{abstract}


%


\section{Introduction}

Statistical learning theory has been studied for general function estimation from data since the late 1960's \cite{vapnik1998statistical}. However, it was only widely adopted in practice after the introduction of the learning algorithms known as Support Vector Machines (SVMs) \cite{vapnik1999overview}. Using the so-called \emph{kernel trick}, which replaces dot products between features and model parameters by evaluations of a \emph{kernel} function, SVMs can learn nonlinear relations from training patterns by solving a convex optimization problem \cite{scholkopf2002learning}.

An important variant of the SVM is the Least Squares Support Vector Machine (LS-SVM) \cite{suykens1999least}, which is obtained by making all data points support-vectors. LS-SVM avoids the constrained quadratic optimization step of standard SVMs by replacing the training procedure with one that reduces to solving a system of linear equations, which can be performed via ordinary least squares. 


The first SVM formulation was derived for classification tasks, but it has been readily adapted to tackle regression problems, being usually named Support Vector Regression (SVR) \cite{drucker1997support}. Similarly, the regression counterpart of LS-SVM is the LS-SVR \cite{suykens2002least}.

It is well known that the LS-SVR model is closely related to the kernel ridge regression (KRR) formulation \cite{saunders1998ridge,cristianini2000introduction}. 
As detailed in the forthcoming sections, those models only differ by the inclusion of a bias term. Interestingly, one can find additional relations by recalling that KRR itself can be seen as the \textit{maximum a posteriori} (MAP) solution of a Gaussian process (GP) model \cite{rasmussen2006gaussian}. Standard SVMs \cite{seeger2000bayesian,sollich2002bayesian} and LS-SVMs \cite{gestel2002bayesian} also present a close relation to the MAP solution of a GP classifier. Similar relations have also been explored in regression settings with the SVR model \cite{gao2002probabilistic, chu2004bayesian}. On the other hand, it can be shown that GPs are a limit for Bayesian neural networks (BNNs) with infinitely many hidden neurons with Gaussian priors in their weights \cite{neal1995bayesian}. The kernel functions related to such limit was derived in \cite{williams1997computing} for standard multilayer perceptron (MLP) with sigmoid activation functions and Gaussian radial basis functions (RBF) neural networks.

Many works have explored theoretical relations between learning models. 
For instance, Bayesian interpretations for SVMs are proposed in \cite{zhu2014bayesian} as an effort to push forward the interface between large-margin learning and Bayesian nonparametrics. Recently, authors in \cite{flam2017mapping} have presented a method for mapping a GP prior to the weights' prior of a BNN. The authors emphasize how their proposal is able to bring properties of a variety of kernels without explicitly considering the weights parameter space. In \cite{garriga2018deep}, it is presented how a residual convolutional neural network with an appropriate prior over the weights and biases behave similarly to a GP. They use such correspondence to introduce a convolutional GP kernel with few hyperparameters. The authors in \cite{lee2018deep} have also tackled the relation between deep neural networks and GPs, deriving an equivalent kernel function for the latter. Those works motivate us to explore further the missing links between learning approaches.

Given the previous observations, our work aims to explore the relations between LS-SVR and Bayesian RBF neural networks. Such theoretical direct connection is not readily available in the literature, and it is required to provide ways to improve least squares based kernel methods with recent advances in the domains of Bayesian methods and neural networks.


Some works, such as \cite{van2001financial,sun2017bayesian}, have already considered the Bayesian formalism in the context of LS-SVM models, but only for the sake of hyperparameter tuning. In \cite{de2010approximate} the authors aim to provide confidence intervals for the model predictions. However, frequentist machinery is used, which makes is dependent on the parameter estimation procedure. The probabilistic interpretation of the LS-SVM classifier presented in \cite{gestel2002bayesian} is the closest work in our direction. However, in that work the authors propose a Bayesian inference framework for the model bias and weights in the LS-SVM primal space. An inference procedure for the model hyperparameters is also presented in \cite{gestel2002bayesian}. As we will detail later, our main result considers the task of performing inference in the variables of the LS-SVR dual space, in the regression setting. 

Thus, in the present work, instead of only considering some evidence-based inference step, we tackle the regression setting with LS-SVR using a RBF kernel and directly relate its model formulation with the RBF neural network within the Bayesian view. Fig. \ref{FIG:MODELS_DIA} illustrates the aforementioned learning approaches connections, as well as the one we cover in the remaining sections.

As follows we summarize both models and derive relations between their formulations. Afterwards, we perform computational experiments to empirically assess our theoretical results. Finally, we draw concluding remarks and discuss the implications of our findings.

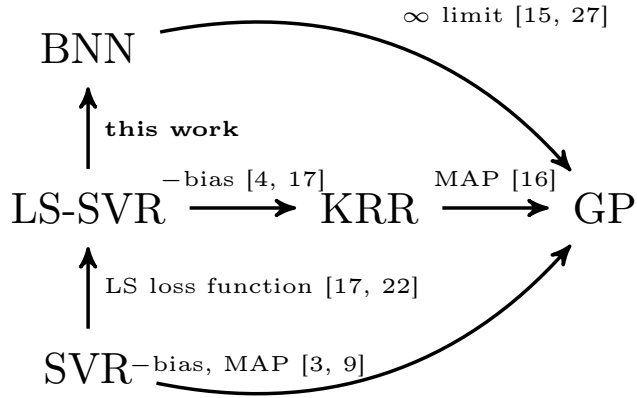
\begin{figure}[!t]
	\centering
\resizebox{0.75\linewidth}{!}{
\begin{tikzpicture}[node distance=1cm, auto,]
 \node (lssvr) {LS-SVR};
 \node[below=of lssvr, yshift=0.2cm] (svr) {SVR}
   edge[pil, ->] node[auto, right] {\tiny LS loss function \cite{saunders1998ridge,suykens2002least}} (lssvr);
 \node[right=of lssvr] (krr) {KRR}
   edge[pil, <-] node[auto, above] {\tiny $-$bias \cite{saunders1998ridge,cristianini2000introduction}} (lssvr);
 \node[right=of krr] (gp) {GP}
   edge[pil, <-] node[auto, above] {\tiny MAP \cite{rasmussen2006gaussian}} (krr)
   edge[pil, <-, bend right=-30] node[auto, above, pos=0.8] {\tiny $-$bias, MAP \cite{gao2002probabilistic, chu2004bayesian}} (svr);
 \node[above=of lssvr, yshift=-0.2cm] (bnn) {BNN}
   edge[pil, <-] node[auto] {\tiny \textbf{this work}} (lssvr)
   edge[pil, bend right=-30] node[auto] {\tiny $\infty$ limit \cite{neal1995bayesian, williams1997computing}} (gp);
\end{tikzpicture}
 }
	\caption{Diagram of relations between the regression learning models considered in the present study. Relevant references are highlighted in each edge.}
	\label{FIG:MODELS_DIA}
\end{figure}

\section{Preliminaries}

In the subsequent subsections, we review the basics of LS-SVR and RBF networks while setting up relevant notation. 

\subsection{LS-SVR in a nutshell}

Let a dataset be composed by $N$ inputs $\bm{x}_i \in \mathbb{R}^D, \forall i \in \{1, \cdots, N\}$, and their correspondent outputs  $y_i \in \mathbb{R}$. We can map the inputs into the outputs using some unknown nonlinear function $f$:
\begin{equation}
f(\bm{x}_i) = \langle \bm{w},\varphi(\bm{x}_i) \rangle + b,~~\mbox{with}~~\bm{w} \in \mathbb{R}^Q, b \in \mathbb{R},
\end{equation}
where $\langle \cdot,\cdot  \rangle$ denotes the dot-product, $\bm{w}$ is a vector of weights, $b$ is a bias and $\varphi(\cdot) : \mathbb{R}^D \rightarrow \mathbb{R}^Q$ is a nonlinear map into some $Q$-dimensional feature space. 
In the LS-SVR framework, we consider the minimization of the following functional~\cite{saunders1998ridge,suykens2002least}: 
\begin{equation}\label{opti_lssvr}
J(\bm{w},e) = \frac{1}{2}\Vert\bm{w}\Vert^2_2 + \gamma \frac{1}{2}\sum_{i=1}^{N}e^{2}_{i},
\end{equation}
subject to the equality constraints
\begin{equation}\label{lssvr_constraints}
y_i = \langle \bm{w},\varphi(\bm{x}_i) \rangle + b + e_i,~~~i=1,2,\ldots,N,
\end{equation}
where $e_i=y_i-f(\bm{x}_i)$ is the error due to the $i$-th input pattern and $\gamma>0$ is a regularization parameter. 

The related Lagrangian function of the optimization problem is
\begin{equation}
L(\bm{w},b,\bm{e},\boldsymbol{\alpha}) =  \frac{1}{2}\Vert\bm{w}\Vert^2_2 + \gamma \frac{1}{2}\sum_{i=1}^{N}e^{2}_{i} -  \sum_{i=1}^{N}\alpha_i [\langle\bm{w},\varphi(\bm{x}_i)\rangle + b + e_i - y_i]
\end{equation}
where the $\alpha_i$ are the Lagrange multipliers. The Karush-Kuhn-Tucker  optimality conditions are given by
\begin{equation}
\left\{
\begin{array}{ll}
\frac{\partial L}{\partial \bm{w}} & = 0~~\rightarrow~~\bm{w}=\sum_{i=1}^{N}\alpha_i\varphi(\bm{x}_i), \\
\frac{\partial L}{\partial e_i} & = 0~~\rightarrow~~\sum_{i=1}^{N}\alpha_i=0, \\
\frac{\partial L}{\partial b} & = 0~~\rightarrow~~\alpha_i=\gamma e_i, \\
\frac{\partial L}{\partial \alpha_i} & = 0~~\rightarrow~~\langle \bm{w},\varphi(\bm{x}_i)\rangle + b + e_i - y_i = 0.
\end{array}
\right.
\end{equation}
The optimal solution $\bm{\theta}_{\text{LS}}$ depends only on $\bm{\theta} = \begin{bmatrix} b \\ \bm{\alpha} \end{bmatrix}$ and is obtained by solving the following linear system:
\begin{equation}\label{eq:lssvr_sys}
\Psi
\begin{bmatrix}
  b \\  \bm{\alpha}
\end{bmatrix} =
\begin{bmatrix}
  0 \\  \bm{y}
\end{bmatrix}, \quad \text{where } \Psi = \begin{bmatrix}
  0~~~~     & \mathbf{1}^\top  \\ 
\mathbf{1}~~~~ & \bm{\Omega} + \gamma^{-1}I
\end{bmatrix}.
\end{equation}
The solution if given by:
\begin{equation}
\label{eq_linear_solution}
\bm{\theta}_{\text{LS}} = \Psi^{-1} \begin{bmatrix}
  0 \\  \bm{y}
\end{bmatrix}.     
\end{equation}
In Eq. \ref{eq:lssvr_sys} we have defined the vectors $\bm{y}=[y_1,y_2,\ldots,y_n]^\top$, $ \bm{1} = [1,1,\ldots,1]^\top$, $\bm{\alpha}=[\alpha_1,\alpha_2,\ldots,\alpha_n]^\top$. We have also defined the kernel matrix $\bm{\Omega} \in \mathbb{R}^{N \times N}$, whose entries are $\Omega_{i,j}=k(\bm{x}_i,\bm{x}_j) = \langle \varphi(\bm{x}_i),\varphi(\bm{x}_j)\rangle$. The so-called kernel function $k(\cdot, \cdot)$ is usually chosen to be a RBF (or Gaussian), expressed by
\begin{equation*}
k(\bm{x}_i, \bm{x}_j) = \exp\{ - c^2 \|\bm{x}_i - \bm{x}_j \|^2 \},    
\end{equation*}
where $\|\cdot\|$ represents the Euclidean norm and $c$ is a kernel hyperparameter.

Predictions for an input $\bm{x}$ with the LS-SVR model can be performed using the optimized $\bm{\alpha}$ and $b$ values:
\begin{equation}\label{aproxi_function}
f(\bm{x}) = \sum_{i=1}^{N}\alpha_ik(\bm{x},\bm{x}_i) + b.
\end{equation}

\subsection{RBF networks in a nutshell}

Radial basis function networks are neural models with a simple structure \cite{haykin2009neural}: a given input pattern $\bm{x} \in \mathbb{R}^D$ is nonlinear transformed by a hidden layer where each unit $i$ is given by
\begin{equation*}
h_i(\bm{x}) = k(\| \bm{x} -  \bm{x}_i \|).
\end{equation*}
The model output is obtained by a linear transformation:
\begin{equation}
\label{EQ_RBF}
f(\bm{x}) = \sum_{i=1}^{N}\alpha_i h_i(\bm{x}) + b,
\end{equation}
where $\alpha_i$ is the $i$-th weight of the output layer and $b$ is the bias term. Although some classic RBF formulations do not present a bias term \cite{haykin2009neural}, many authors do include it in the output layer \cite{chen1991orthogonal}. We note that the latter case is more general, since we could make $b = 0$.

If we choose $h_i(\bm{x}) = k(\| \bm{x} -  \bm{x}_i \|) = \exp\{ - c^2 \|\bm{x} - \bm{x}_i \|^2 \}$, the output of the RBF network would be the same as in Eq. \ref{aproxi_function}. Note that one could opt to build a hidden layer with $M < N$ hidden units, due to inherent data redundancies. In that case, a subset of the training data or the centroids found by a clustering method, e.g., the K-means algorithm, could be used as the centers of the hidden layer \cite{haykin2009neural}.

From Eq. \ref{EQ_RBF} we can see that the optimal solution for $\begin{bmatrix}  b \\  \bm{\alpha} \end{bmatrix}$ is obtained from the following linear system:
\begin{equation}
\label{eq:linear_rbf}
\psi
\begin{bmatrix}
  b \\  \bm{\alpha}
\end{bmatrix} = \bm{y}, \quad \text{where } \psi = [\bm{1} \quad \bm{\Omega}],
\end{equation}
where we have defined the matrix $\bm{\Omega} \in \mathbb{R}^{N \times N}$, whose entries are computed by $\Omega_{i,j}=k(\| \bm{x}_i -  \bm{x}_j \|)$. The solution $\bm{\theta}_{RBF}$ is given by:
\begin{equation}
\label{eq_rbf_solution}
\bm{\theta}_{\text{RBF}} = (\psi^\top \psi)^{-1} \psi^\top \bm{y}.     
\end{equation}

In the present paper we are interested in the Bayesian treatment of the RBF network. In that formulation, we begin by defining the observation model
\begin{equation}
p( y_i |\bm{x}_i , \bm{\theta}) = \mathcal{N}(y_i | f(\bm{x_i}),  \sigma^2),
\end{equation} 
where $f(\bm{x},\bm{\theta}|\mathcal{D})$ is the same as in the right side of Eq. \ref{EQ_RBF}.
We complete the model by placing a prior distribution $p(\bm{\theta})$ on the parameters $\bm{\theta} = ( b , \bm{\alpha} )^\top$ and proceed to compute the posterior given the training data:
\begin{equation}
\label{eq:posterior}
p(\bm{\theta}|\mathcal{D}) = \frac{ p(\mathcal{D}|\bm{\theta}) p(\bm{\theta})}{ \int p(\mathcal{D}|\bm{\theta}) p(\bm{\theta}) \mathrm{d} \bm{\theta} } \propto p(\mathcal{D}|\bm{\theta}) p(\bm{\theta}) 
\end{equation}
 Once in possession of the posterior $p(\bm{\theta}|\mathcal{D})$, we can compute the distribution of the network output $y$ for an input $\bm{x}$ \cite{neal1995bayesian}:
\begin{equation}
p(y | \bm{x}, \mathcal{D}) = \int p(y | \bm(x), \bm{\theta}) p(\bm{\theta} | \mathcal{D})  \mathrm{d}\bm{\theta}, 
\end{equation}
for which the mean, i.e., the expected network output, is given by:
\[
\mathbb{E}\{y\} = \int y\, p(y | \bm{x}, \mathcal{D}) \mathrm{d} y.
\]
 
We note that one could also place priors on kernel parameters, on the variance $\sigma^2$ and on the centers locations in the hidden layer \cite{barber1998radial}, but for the following sections choosing only $p(\bm{\theta})$ is enough.

\section{Main result}
\label{sec:main_result}

We now present a result that shows an equivalence in limit between the estimation procedure of the LS-SVR and the MAP estimation of RBF networks. For this purpose, we introduce the concept of a Bayesian $\epsilon$-LS-SVR.

\begin{definition}[Bayesian $\epsilon$-LS-SVR] 
Given $\epsilon>0$, a Bayesian $\epsilon$-LS-SVR is a Bayesian RBF network model defined using all training points as centroids in the hidden layer and is described as
    \begin{align*}
        &y_n \sim \mathcal{N}\Bigg( \sum_{i=1}^{N} \alpha_i k(\bm{x}_n, \bm{x}_i) + b,\, \sigma^2 = 1 \Bigg), \\
        & [b, \bm{\alpha}]^\top \sim \mathcal{N}(\mu, \Sigma),
    \end{align*}
    where the covariance matrix $\Sigma$ is such that
    \begin{equation*}
    \Sigma^{-1} =
    \begin{bmatrix}
    \epsilon      & \gamma^{-1} \mathbf{1}^\top   \\ 
    \gamma^{-1} \mathbf{1} & \mathbf{1} \mathbf{1}^\top +  2 \gamma^{-1} \Omega +  \gamma^{-2} I
    \end{bmatrix},
    \end{equation*}
    and the mean vector $\mu$ is given by:
    \begin{equation*}
     \mu = \gamma^{-1} \Sigma
    \begin{bmatrix}
        0 \\
        y
    \end{bmatrix}.    
    \end{equation*}      
\end{definition}

And show that, there is an analytic form for the exact difference between MAP estimates in an  Bayesian $\epsilon$-LS-SVR and its traditional counterpart.

\begin{lemma}
  Given $\epsilon>0$, the parameters $\bm{\theta}_\text{LS}$ of an $LS-SVR$ and the MAP estimates $\mu^\star$ of a Bayesian $\epsilon$-LS-SVR defined under the same aforementioned conditions, the gap $\bm{\theta}_{\text{LS}} - \mu^\star$ is exactly:
  \begin{equation*}
 \frac{ \epsilon} {1 + \epsilon s_1 (\Psi^\top \Psi)^{-1} s_1^\top } (\Psi^\top \Psi)^{-1} s_1 s_1^\top \Psi^{-1} \begin{bmatrix} 0 \\  \bm{y} \end{bmatrix}, 
\end{equation*}
where $s_1$ is a $(N+1)$-dimensional column vector of zeros except for a $1$ in its first element.
\end{lemma}

\begin{proof}
Let $\bm{\theta} = [b, \bm{\alpha}]^\top$ and $\Phi \in \mathbb{R}^{N \times (N+1)}$ with its $i$-th line being denoted as $\bm{\phi}_i$ such that
\[ 
\bm{\phi}_i = [1, k(\bm{x}_i, \bm{x}_1), \ldots, k(\bm{x}_i, \bm{x}_N)]^\top .
\]  
The posterior of the aforementioned model is given by:
\begin{align*}
     p\big(b, \bm{\alpha} \, | \, \mathcal{D}, \sigma^2 = 1/2 \big) &\propto  \mathcal{N}(\bm{\theta} | \mu, \Sigma) \prod_{i=1}^{N} \mathcal{N}(y_i | \bm{\phi}_i^\top \bm{\theta}, \sigma^2)\\ &= \mathcal{N}(\bm{\theta} | \mu^\star, \Sigma^\star),
\end{align*}
with parameters $\mu^\star$ and $\Sigma^\star$:
\[
\Sigma^\star = \Bigg( \Sigma^{-1} + \Phi^\top \Phi \Bigg)^{-1} \text{ and } \mu^\star = \Sigma^{\star} \Bigg(\Phi^\top \bm{y}  +  \gamma^{-1}
    \begin{bmatrix}
        0 \\
        \bm{y}
    \end{bmatrix}     \Bigg).
\]
As the mean and mode of a multivariate normal random variable coincide, $\mu^\star$ is also the MAP estimate for $\bm{\theta} = [b, \alpha]^\top$. Note that:
\begin{align*}
\mu^\star = \Sigma^\star 
\begin{bmatrix}
  0~~~~     & \mathbf{1}^\top  \\ 
\mathbf{1}~~~~ & \boldsymbol{\Omega} + \gamma^{-1}I
\end{bmatrix}^\top
\begin{bmatrix}
  0 \\  \bm{y}
\end{bmatrix} =
\Sigma^\star 
\Psi^\top
\begin{bmatrix}
  0 \\  \bm{y}
\end{bmatrix}.
\end{align*}
We can also write $\Psi^\top \Psi$ as
\begin{equation*}
    \begin{bmatrix}
  N     & \mathbf{1}^\top \Big(\Omega  + \gamma^{-1} I \Big) \\ 
 \Big(\Omega  + \gamma^{-1} I \Big) \mathbf{1} & \mathbf{1} \mathbf{1}^\top + \Big(\Omega + \gamma^{-1} I\Big)^\top \Big(\Omega + \gamma^{-1} I\Big) 
\end{bmatrix}.
\end{equation*}

Observe that $\Phi$ is essentially the kernel matrix $\Omega$ padded with ones on the left, i.e.:
\[
    \Phi = \begin{bmatrix}
      \mathbf{1} & \Omega
    \end{bmatrix},
\]
so that:
\[
    \Phi^\top \Phi = \begin{bmatrix}
      \mathbf{1}^\top \\
      \Omega^\top
    \end{bmatrix} \begin{bmatrix}
      \mathbf{1} & 
      \Omega
    \end{bmatrix}
    = \begin{bmatrix}
      N & \mathbf{1}^\top \Omega\\
      \Omega^\top \mathbf{1} & \Omega^\top \Omega
    \end{bmatrix}.
\]
Notice that the matrix $\Sigma^{-1} $ is given by
\[
\Sigma^{-1}  =
 \Psi^\top \Psi + \begin{bmatrix}
   \sqrt{\epsilon}\\
   0\\
   \vdots\\
   0
 \end{bmatrix}  \begin{bmatrix}
   \sqrt{\epsilon}, 0, \ldots, 0
 \end{bmatrix} - \Phi^\top \Phi, 
\]
and thus, the matrix $(\Sigma^{\star})^{-1}$ becomes
\[
 (\Sigma^{\star})^{-1} =
 \Sigma^{-1} + \Phi^\top \Phi  =
 \Psi^\top \Psi + \begin{bmatrix}
   \sqrt{\epsilon}\\
   0\\
   \vdots\\
   0
 \end{bmatrix}  \begin{bmatrix}
   \sqrt{\epsilon}, 0, \ldots, 0
 \end{bmatrix},
\]
which can we rewritten as:
\[
 (\Sigma^{\star})^{-1} =
 \Psi^\top \Psi + \epsilon s_1 s_1^\top,
\]
where $s_1$ is a $(N+1)$-dimensional column vector of zeros except for a $1$ in its first element.

Applying Sherman-Morrison's matrix inversion identity on $(\Sigma^\star)^{-1}$ and rearranging the terms we get the expression below:
\begin{equation*}
 \Sigma^{\star} = (\Psi^\top \Psi )^{-1} -  \frac{ \epsilon} {1 + \epsilon s_1 (\Psi^\top \Psi)^{-1} s_1^\top } (\Psi^\top \Psi)^{-1} s_1 s_1^\top (\Psi^\top \Psi)^{-1}.
\end{equation*}
Thus, we can rewrite the vector $\mu^\star$ as
\begin{equation*}
\mu^\star = \bigg\{(\Psi^\top \Psi)^{-1} \Psi^\top \left. - \frac{ \epsilon} {1 + \epsilon s_1 (\Psi^\top \Psi)^{-1} s_1^\top } (\Psi^\top \Psi)^{-1} s_1 s_1^\top (\Psi^\top \Psi)^{-1} \Psi^\top \right\} \begin{bmatrix}
  0 \\  \bm{y}
\end{bmatrix}.
\end{equation*}
Finally, we are able to compute the gap $\bm{\theta}_{\text{LS}} - \mu^*$:
\begin{equation*}
\begin{split}
& \bm{\theta}_{\text{LS}} - \mu^* = \bigg\{ \Psi^{-1} - (\Psi^\top \Psi)^{-1} \Psi^\top \\
& \left. + \frac{ \epsilon} {1 + \epsilon s_1 (\Psi^\top \Psi)^{-1} s_1^\top } (\Psi^\top \Psi)^{-1} s_1 s_1^\top (\Psi^\top \Psi)^{-1} \Psi^\top \right\} \begin{bmatrix}
  0 \\  \bm{y}
\end{bmatrix}.        
\end{split}
\end{equation*}
Since $\Psi$ is square and its inverse $\Psi^{-1}$ exists, which is true given its definition in Eq. \eqref{eq:lssvr_sys}, we get $\Psi^{-1} = (\Psi^\top \Psi)^{-1} \Psi^\top$ and we obtain the desired result.
\end{proof}

Analyzing this gap in the limit as $\epsilon$ shrinks, we can derive the following equivalency result, which incurs directly in the subsequent corollary.

\begin{thm}
    As  $\epsilon$ decreases towards zero, the parameters $\bm{\theta}_\text{LS}$ of LS-SVR coincide with the    MAP estimate $\mu^\star$ of the Bayesian $\epsilon$-LS-SVR defined over the same dataset with same hyper-parameters. 
\end{thm}

 \begin{proof} Using Lemma 1, as $\epsilon \xrightarrow[]{+} 0$, the gap $\bm{\theta}_{\text{LS}} - \mu^\star$ tends to
\begin{align*}
 \lim_{\epsilon \xrightarrow[]{+} 0} \frac{ \epsilon} {1 + \epsilon s_1 (\Psi^\top \Psi)^{-1} s_1^\top } (\Psi^\top \Psi)^{-1} s_1 s_1^\top \Psi^{-1} \begin{bmatrix}
  0 \\  \bm{y}
\end{bmatrix},
\end{align*}
which equals zero. Thus, as $\epsilon$ decreases towards zero, the gap  $\bm{\theta}_{\text{LS}} - \mu^\star$ vanishes and $\mu^\star$ coincides with the LS-SVR estimates.
\end{proof}

\begin{corol}
    For any $\bm{x} \in \mathbb{R}^D$, as $\epsilon$ decreases towards zero, the LS-SVR prediction using Equation \ref{aproxi_function} coincides with the ones obtained using the same rule but using the MAP estimates of the respective Bayesian $\epsilon$-LS-SVR.
\end{corol}


\section{Discussion and Limitations}

The results detailed in Section \ref{sec:main_result} lay down the theoretical similarities of the LS-SVR model with the RBF kernel and the Bayesian RBF network. We were able to show how the outputs of both models can be matched, which itself depends on the same parameter inference procedure. Our probabilistic presentation, for instance, incorporates the $\epsilon$ hyperparameter in the introduced $\epsilon$-LS-SVR formulation, which corresponds to the precision of the prior given to the model bias. In the limit $\epsilon \rightarrow 0$, the bias has an improper uniform prior.


As mentioned in the Introduction, differently from \cite{gestel2002bayesian} for the LS-SVM, we consider a Bayesian approach directly in the dual formulation of the LS-SVR. We emphasize that this choice is comprehensive enough and sufficient for the proposed analysis, given that kernel-based methods in general apply only dual variables to perform their learning and predictions. 
Furthermore, although not explored in the current work, the hyperparameter inference level pursued in \cite{gestel2002bayesian} could be readily considered by choosing a (hyper)prior for $\epsilon$ and the regularization hyperparameter $\gamma$. 
However, the effective number of parameters analysis performed in \cite{gestel2002bayesian} is not straightforwardly ported to our approach, since we bypass the primal formulation. The latter also prevents us to make more detailed geometric interpretations in the primal space.

One could question if the similarities between the expressions we have derived actually translate to model equivalence.
%
Full model equivalence is a concept difficult to define. However, we have demonstrated that both formulations have the same parameter estimation problem, the same (Bayesian) model selection approach and the same predictive rules. Although we have strayed from the support vectors analogy usually defined in the primal weight formulation, which causes the above mentioned interpretation limitations, the lack of such analogy does not disqualify our claims.

Finally, while our derivations implied the use of the RBF kernel, the standard LS-SVR framework allows for any valid kernel function to be used. 
Even though the choice of a non-RBF kernel causes the RBF network correspondence to be broken, it is important to note that the resulting variant would present similar inference expressions and be comparable to a Bayesian generalized linear model \cite{gelman2013bayesian}. 
A more thorough analysis of the latter implication is left for future investigations.

\section{Illustrative experiments} 

In this section we report some empirical results to illustrate the presentation detailed so far. First, we show how the MAP estimate for the $\epsilon$-LS-SVR can get arbitrarily close to the standard LS-SVR weight vector solution $\bm{\theta}_\text{LS}$ as we choose small $\epsilon$ values. Then, we perform a simple experiment to show how to make Bayesian predictions with the formulated $\epsilon$-LS-SVR model.

\subsection{Comparison of the $\epsilon$-LS-SVR and LS-SVR models}

\begin{figure}[t] \centering
    \begin{tikzpicture}[scale=0.45] 
        \begin{axis}[xmin=0,xmax=5, ymin=0, ymax = 0.07, ytick = { 0.02, 0.04, 0.06, 0.07}, yticklabels = { 2, 4, 6, 7}, legend pos= north east,  grid=both, title=MPG, , xticklabel style = {font=\Large}, yticklabel style = {font=\Large},  title style = {font=\LARGE},every axis plot/.append style={ultra thick}, ylabel = ARSE, ylabel style = {font=\huge, at={(-0.1,0.5)}}]
        \addplot[color=blue,mark=pentagon*, mark options={color=magenta}] table [x=Missing, y=RMSE, col sep=comma, row sep=\\, title=Compression] {
          Missing,	RMSE\\
            0      , 0.0628   \\
            0.5000 , 0.0381   \\
            1.0000 , 0.0232   \\
            1.5000 , 0.0140   \\
            2.0000 , 0.0085  \\
            2.5000 , 0.0052  \\
            3.0000 , 0.0031  \\
            3.5000 , 0.0019  \\
            4.0000 , 0.0011  \\
            4.5000 , 0.0007 \\
            5.0000 , 0.0004 \\};             
        \end{axis}
    \end{tikzpicture}  
    \begin{tikzpicture}[scale=0.45] 
        \begin{axis}[ xmin=0, xmax = 5, ymin = 0,ymax = 0.04, ytick = { 0.01, 0.02, 0.03, 0.04}, yticklabels = { 1, 2, 3, 4}, legend pos= north east,  grid=both, title=Compression, xticklabel style = {font=\Large}, yticklabel style = {font=\Large},title style = {font=\LARGE},every axis plot/.append style={ultra thick}]
        \addplot[color=blue,mark=pentagon*, mark options={color=magenta}] table [x=Missing, y=RMSE, col sep=comma, row sep=\\] {
            Missing,	RMSE\\
            0      , 0.0399669296336829   \\
            0.5000 , 0.0242626647407492   \\
            1.0000 , 0.0147239694594507   \\
            1.5000 , 0.00893345481733988  \\
            2.0000 , 0.00541948750910657  \\
            2.5000 , 0.00328748029117789  \\
            3.0000 , 0.00199410291524305  \\
            3.5000 , 0.00120953802530745  \\
            4.0000 , 0.000733641567211602 \\
            4.5000 , 0.000444983340628700 \\
            5.0000 , 0.000269898701652869 \\ };    
        \end{axis}
    \end{tikzpicture} \\
    \hspace{0.25in} $-\log\epsilon$\\
    \begin{tikzpicture}[scale=0.45] 
        \begin{axis}[ xmin=0, xmax = 5, ymax = 0.07, ytick = {0.02, 0.04, 0.06, 0.07}, yticklabels = { 2, 4, 6, 7},ymin=0, legend pos= north east,  grid=both,title=Stocks , xticklabel style = {font=\Large},  yticklabel style = {font=\Large},title style = {font=\LARGE},every axis plot/.append style={ultra thick}, ylabel = ARSE, ylabel style = {font=\huge, at={(-0.1,0.5)}}]]
        \addplot[color=blue,mark=pentagon*, mark options={color=magenta}] table [x=Missing, y=RMSE, col sep=comma, row sep=\\] {
          Missing,	RMSE\\
            0      , 0.0615837721347982   \\
            0.5000 , 0.0374459911049471   \\
            1.0000 , 0.0227466936998484   \\
            1.5000 , 0.0138093092102190   \\
            2.0000 , 0.00838046392458909  \\
            2.5000 , 0.00508473688337742  \\
            3.0000 , 0.00308468506396798  \\
            3.5000 , 0.00187119021007842  \\
            4.0000 , 0.00113502038571168  \\
            4.5000 , 0.000688456363448130 \\
            5.0000 , 0.000417581552633746 \\};             
        \end{axis}
    \end{tikzpicture} 
    \begin{tikzpicture}[scale=0.45] 
        \begin{axis}[ xmin=0, xmax = 5, ymax = 0.05, ymin=0,ytick = {0.015,0.03,0.045,0.05}, yticklabels = {1.5,3,4.5,5}, legend pos= north east,  grid=both, title=Bostton Housing , xticklabel style = {font=\Large},  yticklabel style = {font=\Large},title style = {font=\LARGE},every axis plot/.append style={ultra thick}]
        \addplot[color=blue,mark=pentagon*, mark options={color=magenta}] table [x=Missing, y=RMSE, col sep=comma, row sep=\\] {
          Missing,	RMSE\\
            0      , 0.0466017844188908   \\
            0.5000 , 0.0283009506357296   \\
            1.0000 , 0.0171784949735541   \\
            1.5000 , 0.0104241092844599   \\
            2.0000 , 0.00632431836599472  \\
            2.5000 , 0.00383654681802085  \\
            3.0000 , 0.00232722386805047  \\
            3.5000 , 0.00141162115269548  \\
            4.0000 , 0.000856224078834130 \\
            4.5000 , 0.000519338137876429 \\
            5.0000 , 0.000314998911544298 \\};             
        \end{axis}
    \end{tikzpicture} \\
    \hspace{0.25in} $-\log\epsilon$
    \caption{Illustration of the convergence of the predictions of the $\epsilon$-LS-SVR and the LS-SVR model as $\epsilon$ decreases towards zero. The vertical axis shows the Average Root Square Error (ARSE) between the predictions obtained by both models.}
\label{fig:illu1}
\end{figure}
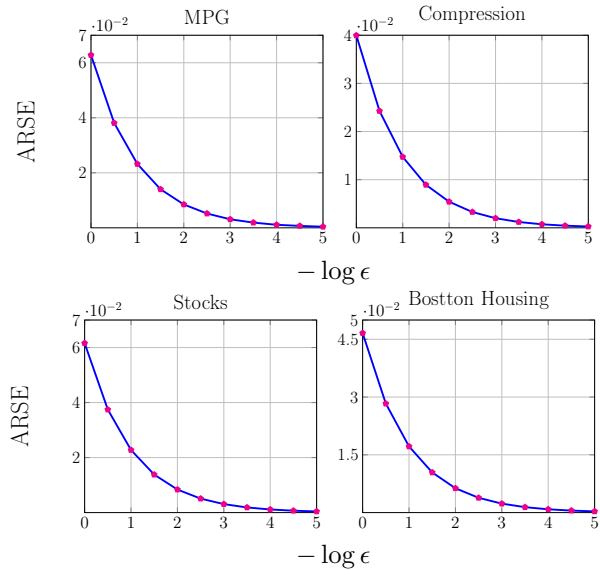

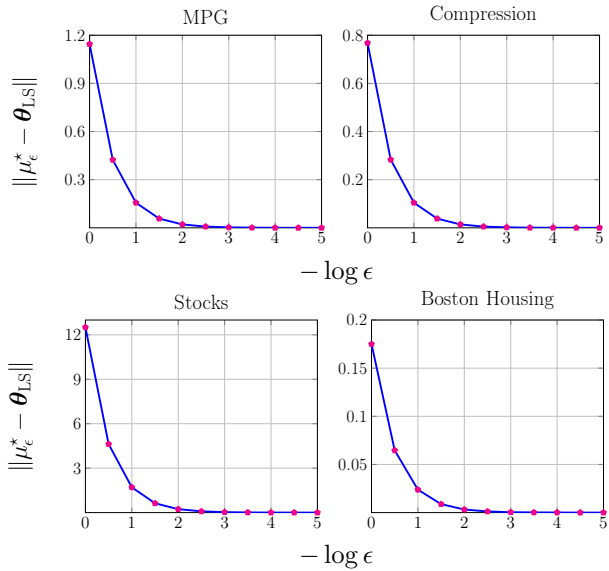
\begin{figure}[t] \centering
    \begin{tikzpicture}[scale=0.45] 
        \begin{axis}[xmin=0,xmax=5, ymin=0, ymax = 1.2, ytick = { 0.3, 0.6, 0.9, 1.2}, yticklabels = { 0.3, 0.6, 0.9, 1.2}, legend pos= north east,  grid=both, title=MPG, , xticklabel style = {font=\Large}, yticklabel style = {font=\Large},  title style = {font=\LARGE},every axis plot/.append style={ultra thick}, ylabel = $\|\mu_\epsilon^\star - \bm{\theta}_\text{LS} \|$, ylabel style = {font=\huge, at={(-0.1,0.5)}}]
        \addplot[color=blue,mark=pentagon*, mark options={color=magenta}] table [x=Missing, y=RMSE, col sep=comma, row sep=\\, title=Compression] {
          Missing,	RMSE\\
            0, 1.1448 \\
            0.5000, 0.4235 \\
            1.0000, 0.1564 \\
            1.5000, 0.0576 \\
            2.0000, 0.0212 \\
            2.5000, 0.0078 \\
            3.0000, 0.0029 \\
            3.5000, 0.0011 \\
            4.0000, 0.0004 \\
            4.5000, 0.0001 \\
            5.0000, 0.0001 \\};             
        \end{axis}
    \end{tikzpicture}  
    \begin{tikzpicture}[scale=0.45] 
        \begin{axis}[ xmin=0, xmax = 5, ymin = 0,ymax = 0.8, ytick = { 0.2, 0.4, 0.6, 0.8}, yticklabels = { 0.2, 0.4, 0.6, 0.8}, legend pos= north east,  grid=both, title=Compression, xticklabel style = {font=\Large}, yticklabel style = {font=\Large},title style = {font=\LARGE},every axis plot/.append style={ultra thick}]
        \addplot[color=blue,mark=pentagon*, mark options={color=magenta}] table [x=Missing, y=RMSE, col sep=comma, row sep=\\] {
            Missing,	RMSE\\
            0, 0.7680 \\
            0.5000, 0.2830 \\
            1.0000, 0.1042 \\
            1.5000, 0.0384 \\
            2.0000, 0.0141 \\
            2.5000, 0.0052 \\
            3.0000, 0.0019 \\
            3.5000, 0.0007 \\
            4.0000, 0.0003 \\
            4.5000, 0.0001 \\
            5.0000, 0.0000 \\ };    
        \end{axis}
    \end{tikzpicture} \\
    \hspace{0.25in} $-\log\epsilon$\\
    \begin{tikzpicture}[scale=0.45] 
        \begin{axis}[ xmin=0, xmax = 5, ymax = 13, ytick = {3,6,9,12}, yticklabels = { 3,6,9,12},ymin=0, legend pos= north east,  grid=both,title=Stocks , xticklabel style = {font=\Large},  yticklabel style = {font=\Large},title style = {font=\LARGE},every axis plot/.append style={ultra thick}, ylabel = $\|\mu_\epsilon^\star - \bm{\theta}_\text{LS} \|$, ylabel style = {font=\huge, at={(-0.1,0.5)}}]]
        \addplot[color=blue,mark=pentagon*, mark options={color=magenta}] table [x=Missing, y=RMSE, col sep=comma, row sep=\\] {
          Missing,	RMSE\\
            0, 12.5081 \\
            0.5000, 4.6245  \\
            1.0000, 1.7065  \\
            1.5000, 0.6289  \\
            2.0000, 0.2316  \\
            2.5000, 0.0853  \\
            3.0000, 0.0314  \\
            3.5000, 0.0115  \\
            4.0000, 0.0042  \\
            4.5000, 0.0016  \\
            5.0000, 0.0006 \\};             
        \end{axis}
    \end{tikzpicture} 
    \begin{tikzpicture}[scale=0.45] 
        \begin{axis}[ xmin=0, xmax = 5, ymax = 0.2, ymin=0,ytick = {0.05,0.1,0.15,0.2}, yticklabels = { 0.05,0.1,0.15,0.2}, legend pos= north east,  grid=both, title=Boston Housing , xticklabel style = {font=\Large},  yticklabel style = {font=\Large},title style = {font=\LARGE},every axis plot/.append style={ultra thick}]
        \addplot[color=blue,mark=pentagon*, mark options={color=magenta}] table [x=Missing, y=RMSE, col sep=comma, row sep=\\] {
          Missing,	RMSE\\
            0      , 0.1750 \\
            0.5000 , 0.0646 \\
            1.0000 , 0.0238 \\
            1.5000 , 0.0088 \\
            2.0000 , 0.0032 \\
            2.5000 , 0.0012 \\
            3.0000 , 0.0004 \\
            3.5000 , 0.0002 \\
            4.0000 , 0.0001 \\
            4.5000 , 0.0000 \\
            5.0000 , 0.0000	\\};             
        \end{axis}
    \end{tikzpicture} \\
    \hspace{0.25in} $-\log\epsilon$
    \caption{Illustration of the convergence of $\mu_\epsilon^\star$ towards $\bm{\theta}_\text{LS}$ as $\epsilon$ decreases towards zero. The vertical axis shows the Euclidean distance between $\bm{\theta}_\text{LS}$ and the MAP parameters estimate $\mu_\epsilon^{\star}$ of the $\epsilon$-LS-SVR.}
\label{fig:illu2}
\end{figure}

To empirically illustrate our main theoretical result, we performed a set of experiments using $4$ datasets a available at UCI machine learning repository \cite{Dua:2019}. We progressively decrease the $\epsilon$-LS-SVR model $\epsilon$ value and perform MAP estimation for its parameters. After that, we compute the Average Root Square Error (ARSE) between the estimates of both methods, given by $\text{ARSE} = \frac{1}{N}\sum_{i=1}^N \left|\mu_i^{\epsilon\text{-LS-SVR}} - y_i^{\text{LS-SVR}}\right|$, where $\mu_i^{\epsilon\text{-LS-SVR}}$ is the predicted mean of the $\epsilon$-LS-SVR and $y_i^{\text{LS-SVR}}$ is the prediction of the standard LS-SVR. We set $\gamma= 0.5$ and $c^2 = 1$. 

The obtained results are presented in Fig. \ref{fig:illu1}. As can be noticed, the estimates of both methods converge to the same values as $\epsilon$ decreases. A similar result is obtained when comparing the parameters estimated by the two learning strategies. Fig. \ref{fig:illu2} shows the Euclidean distance between the weight vectors obtained by the $\epsilon$-LS-SVR and LS-SVR models. These results coincide with the behaviour outlined in the theorem that we have introduced in Section \ref{sec:main_result}.

\subsection{Bayesian predictions with the LS-SVR model}

\begin{figure}[t]
    \centering
    \includegraphics[width=.75\columnwidth]{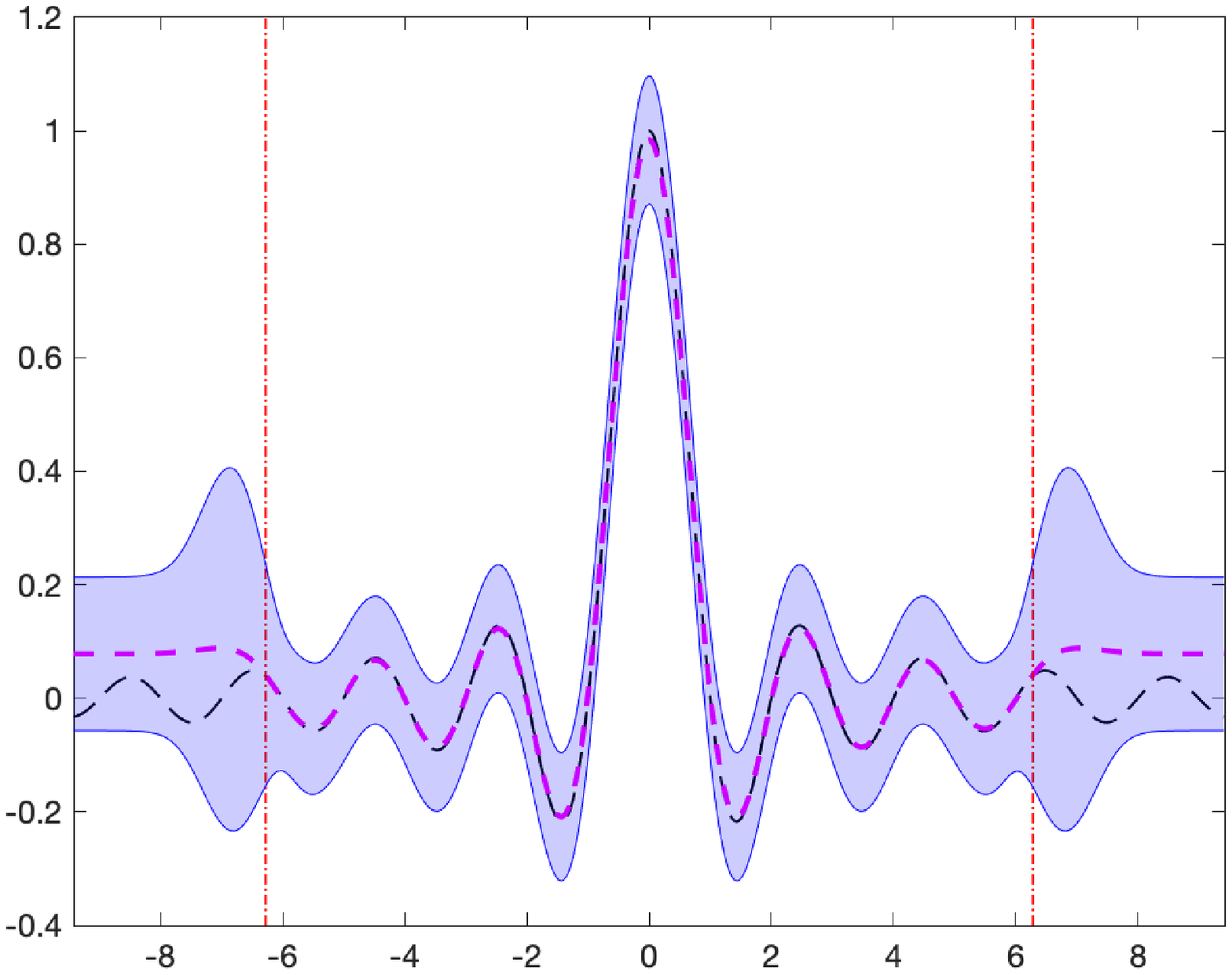}
    \caption{Illustration of the use of $\epsilon$-LS-SVR to predict the univariate (normalized) sinc function, shown as the dashed black  line. The dashed magenta line shows the average prediction of our model and the blue bands around it correspond to $\pm$ one standard deviation. The vertical thin red dashed lines bounds the area in which samples were observed.}
    \label{fig:predictions}
\end{figure}

Once in possession of the $\epsilon$-LS-SVR posterior, one can compute statistics of an arbitrary test function. For example, we can compute the expected value and variance of the LS-SVR prediction rule, in Eq. \eqref{aproxi_function}, as
\begin{align}
    &\E_{\bm{\theta} \,|\, \mathcal{D}, \sigma^2}[f(\bm{x})] = b + \sum_{i=1}^{N} k(\bm{x}, \bm{x}_i) {\mu^{\star}_{i+1}} \label{eq:pred1} ,\\
    &\Var_{\bm{\theta}\,|\, \mathcal{D}, \sigma^2}[f(\bm{x})] =  \sum_{i=1}^{N} \sum_{j=1}^{N} k(\bm{x}, \bm{x}_i) k(\bm{x}, \bm{x}_j) \Sigma^\star_{i, j}. \label{eq:pred2}
\end{align}

We illustrate these capabilities by modeling the univariate (normalized) sinc function:
\[
    y_i = \frac{\sin(\pi x_i)}{\pi x_i}, \quad \forall x_i \in \mathbb{R}.
\]

For this purpose, we created a dataset by taking $1200$  input points uniformly drawn in the range $x \in [-2\pi, 2\pi]$ and evaluating the sinc function in each of them. We used $c=\gamma=1$ and $\epsilon=10^{-4}$. After the $\epsilon$-LS-SVR posteriori has been computed, we used Eqs. \eqref{eq:pred1}--\eqref{eq:pred2} to predict the values of the sinc function for new values in the range $x \in [-3\pi, 3\pi]$. The results are shown in Fig. \ref{fig:illu2}. Note that, as expected, variances are lower in the regions contemplated in the observed data than outside it.  
This experiment has been repeated ten, times yielding similar results.

\section{Conclusion}


We have presented a theorem showing a relationship between LS-SVR and Bayesian RBF networks. In a nutshell, the result provides a way to build a Bayesian RBF network which, when taking MAP estimates, presents theoretical similarities to a LS-SVR model. Such relationship opens a path to improve LS-SVR models with well-established advances in Bayesian methodology.

As a byproduct, we have introduced $\epsilon$-LS-SVR, a very particular Bayesian RBF network, which outlines new foundations for building Bayesian LS-SVRs. The posterior from our model can be used, for instance, to inform decision making processes (\textit{e.g}. computing expected utilities) or to provide distributions for a new output $y^\prime$ given a new input vector $\bm{x}^\prime$. 

A natural improvement would be to automate the selection of the regularizer $\gamma^{-1}$, the bias prior precision $\epsilon$ and the observation noise $\sigma^2$. This can be easily achieved by placing priors on these parameters. Besides freeing us from laborious parameter tuning, priors can also contain the subjective knowledge of a domain expert. While the introduction of such priors might break conjugacy, we can still resort to full Bayesian inference (\textit{e.g.}, using Markov Chain Monte Carlo), which would allow us to still recognize our uncertainty about the estimated parameters, in the form of a posterior distribution. Another promising direction would be to transfer information between different $\epsilon$-LS-SVR models using hierarchical models. All those investigations are left for future work.

\bibliography{Bibliography,IEEEabrv}

\end{document}